\documentclass{article}
\usepackage{arxiv}

\usepackage{outlines}
\usepackage[utf8]{inputenc} 
\usepackage[T1]{fontenc}
\usepackage{hyperref}
\usepackage{url}   
\usepackage{booktabs} 
\usepackage{amsfonts}
\usepackage{amsthm}
\usepackage{nicefrac}
\usepackage{microtype}
\usepackage{graphicx}
\usepackage{mathtools}

\usepackage[natbibapa]{apacite}

\usepackage{doi}
\usepackage{makecell}

\usepackage{amsmath}
\usepackage[thinc]{esdiff}
\usepackage{todonotes}

\usepackage{glossaries}
\usepackage{svg}
\usepackage{tabularray}

\title{Predictive Coding as a Neuromorphic Alternative to Backpropagation: A Critical Evaluation
}

\date{}

\author{ 
    Umais Zahid\\
	Huawei Technologies R\&D\\
	London, UK\\
	\texttt{umais.zahid@huawei.com}
    \And
    Qinghai Guo\\
    Huawei Technologies R\&D\\
    Shenzhen, China\\
    \texttt{guoqinghai@huawei.com}
    \And
    Zafeirios Fountas \\
	Huawei Technologies R\&D\\
	London, UK\\
	\texttt{zafeirios.fountas@huawei.com}
}

\date{}

\renewcommand{\shorttitle}{Predictive Coding and Backpropagation}
\renewcommand{\undertitle}{}

\hypersetup{
pdftitle={Predictive coding and backpropagation},
pdfsubject={q-bio.NC, q-bio.QM},
pdfauthor={Umais Zahid, Qinghai Guo, Zafeirios Fountas},
pdfkeywords={Backpropagation, Predictive Coding, Neuromorphic computing},
}

\begin{document}
\maketitle

\begin{abstract}
Backpropagation has rapidly become the workhorse credit assignment algorithm for modern deep learning methods. Recently, modified forms of predictive coding (PC), an algorithm with origins in computational neuroscience, have been shown to result in approximately or exactly equal parameter updates to those under backpropagation. Due to this connection, it has been suggested that PC can act as an alternative to backpropagation with desirable properties that may facilitate implementation in neuromorphic systems. Here, we explore these claims using the different contemporary PC variants proposed in the literature. We obtain time complexity bounds for these PC variants which we show are lower-bounded by backpropagation. We also present key properties of these variants that have implications for neurobiological plausibility and their interpretations, particularly from the perspective of standard PC as a variational Bayes algorithm for latent probabilistic models. Our findings shed new light on the connection between the two learning frameworks and suggest that, in its current forms, PC may have more limited potential as a direct replacement of backpropagation than previously envisioned.
\end{abstract}


\glsaddall

\section{Introduction}

Predictive coding (PC) is a prominent neuroscientific theory that has emerged in the last two decades as a highly compelling computational model of perception and action in the brain \citep{rao_predictive_1999, friston_theory_2005}. From a theoretical standpoint, PC, in its standard formulation, has benefited from its interpretation as a variational Bayes algorithm for learning and inference in latent hierarchical models \citep{friston_learning_2003, friston_theory_2005, friston_predictive_2009, bogacz_tutorial_2017} providing it credence as a plausible instantiation of normative theories such as the Bayesian brain hypothesis, and the free-energy principle. More empirically, its candidate neurobiological implementation has demonstrated impressively close correspondences to the canonical circuitry of the cortex \citep{bastos_canonical_2012, shipp_neural_2016}, while also successfully explaining or reproducing a number of neurophysiological and cognitive phenomena, such as end-stopping \citep{rao_predictive_1999}, binocular rivalry \citep{hohwy_predictive_2008}, attention \citep{feldman_attention_2010, kanai_cerebral_2015}, and biases in the perception of event duration \citep{fountas_predictive_2022}.

More recent work has suggested a connection between modified forms of PC, and backpropagation, suggesting the former as an alternative to the latter with the various benefits one may naturally expect to associate for a theory formulated with neurobiological constraints in mind\citep{millidgePredictiveCodingFuture2022, millidge_activation_2020, song_can_2020}. These include locality of computation and parallelisability that may lend it improved performance, and greater amenability to implementation on neuromorphic hardware.  

The purpose of this work is to investigate these suggestions and clarify more precisely the relationship between backpropagation and PC, particularly with regards to the kinds of assumptions that are necessary to enable such a comparison, and furthermore, to define exactly what advantage PC-based formulations may have with respect to both parallelisability, locality and ultimately, practical performance. In what follows, we re-review some of the ground covered by \citep{rosenbaum_relationship_2022}, while extending it to discuss the implications for PC in general with respect to computational and memory complexity, as well as compute and memory locality. We also present proofs for lower bounds on the time complexity of current PC variants relative to backpropagation, and validate these proofs empirically. Our work finds that, while the examined algorithms present interesting, biologically-inspired alternative implementations of backpropagation, they are provably guaranteed to be slower under identical problem settings, and that it is unclear to what degree they improve upon locality, particularly in comparison to modern candidate proposals for backpropagation. Lastly, we note that when adopting the modifications required to enable such comparisons, PC models diverge from the generative variational Bayes framework, frequently used as the motivational basis of this theory.

\section{Backpropagation}
\label{sec:backpropagation}

Backpropagation, first introduced by \citep{linnainmaa_taylor_1976} in the context of accumulating rounding errors, and later popularised by \citep{rumelhart_learning_1986} in the context of neural networks, has quickly become the workhorse credit assignment algorithm for modern deep learning tasks. Described most simply, backpropagation can be seen as an algorithmically efficient implementation of the chain rule from calculus that allows the computation of high dimensional gradients for deeply nested functions with respect to scalar or low-dimensional outputs. It accomplishes this by defining gradients of values deep inside a nested function in terms of the gradients of their children (their functional dependants). The resultant recursive relationship is then evaluated for values closest to the output, sequentially backward. This recursive relationship is what enables backpropagation to efficiently compute gradients with respect to all parameters in the function chain with a single forward and backward computational pass. \\

To illustrate how this operates, we present the following simple but highly generic problem setting. Let us consider an arbitrary function $F$, defined by the composition of many constituent functions $f_i$, each optionally parameterised by some set of parameters $\mathbf{\theta_i}$. We may then write the following: 
\begin{align}
\label{eq:forward}
   f_0(x_0, \theta_0) &= \mu_1 \nonumber \\
   f_1(\mu_1, \theta_1) &= \mu_2 \nonumber \\
   ... \\
   f_{(L-1)}(\mu_{L-1}, \theta_{L-1}) &= \mu_{L} \nonumber \\
   f_{L}(\mu_L, x_L) &= E \nonumber 
\end{align}
where we use $\mu_i$ to denote the intermediary outputs of our composite function, and $x_0, x_L$ to denote known values which are provided as inputs and outputs of our function respectively. $x_0$ and $x_L$ may for example be inputs and classification labels respectively, $\mu_\ell$ may be intermediate activation vectors for a multilayer perceptron, and $f_L$ a classification loss. 

If we wish to compute the gradient of any intermediate parameters with respect to our function output $E$, we may then do so by first defining the following simple recursive relationship. 
\begin{align}
\label{eq:recursive_backprop_equations}
e_\ell = \left.
\begin{cases}
    \frac{\partial f_{\ell}}{\partial \mu_{\ell}}^T
    e_{\ell+1}
    & \text{for } \ell=1,\dots,L-1 \\[5pt]  
    \frac{\partial f_L}{\partial \mu_L}
    & \text{for } \ell=L 
\end{cases}
\right\}
\end{align}

For multivariate elementary functions, the $\frac{\partial f_{\ell}}{\partial \mu_L}$ terms would correspond to the Jacobian matrices for the elementary functions $f_i$ with respect to $\mathbf{x_i}$. And by virtue of the chain rule, the error terms $e_{\ell}$ would thus equal the gradient of E with respect to $x_{\ell+1}$. 

Having established error vectors $e_{\ell}$ that correspond to the gradient with respect to $x_\ell$, one may then compute the gradient of output E with respect to an arbitrary intermediate parameter by matrix-multiplying the error vector from the next layer by the Jacobian of the function with respect to the parameters.
\begin{align}
\label{eq:parameter_backprop_equation}
    \frac{\partial E}{\partial \mathbf{\theta_{\ell}}} = -\frac{\partial \mu_{\ell+1}}{\mathbf{\theta_{\ell}}}^T e_{\ell+1} 
\end{align}

Though obvious in this scenario, it will be important for our later comparisons to mention here that all Jacobians mentioned above are evaluated at the values of the feed-forward outputs of our function, which, as we will discuss, may not necessarily be the case for Bayesian PC. We also note that, in practice, the Jacobians specified in Equations (\ref{eq:parameter_backprop_equation}) and (\ref{eq:recursive_backprop_equations}) are almost never instantiated explicitly. Rather, an automatic differentiation program that implements backpropagation would associate each elementary or constituent function with an associated vector-Jacobian product (VJP) function, which is generally far more computationally and memory efficient \citep{paszke_automatic_2017, bradbury_jax_2018}.  

\subsection{Computational and Time Complexity}

The cost of computing the forward pass (Equation \ref{eq:forward} and subsequent backward recursions  (Equations \ref{eq:recursive_backprop_equations} and \ref{eq:parameter_backprop_equation}) can be shown to have computational and time cost bounded by a small constant multiple of the cost for a single forward pass through the corresponding elementary function. This well-known result from the automatic differentiation literature is sometimes called the cheap-gradient principle and is one reason why backpropagation has been so successful as a credit assignment algorithm in modern large data settings. This constant was shown to be 3 for rational functions in the seminal work of \citep{baur_complexity_1983}, and 5 more generally for any function composed of elementary arithmetic and trigonometric functions \citep{griewank_automatic_1997, griewank_evaluating_2008, griewank_complexity_2009}. The precise value of this constant generally also depends upon details regarding the relative costs of various basic mathematical operations (such as addition, multiplication, memory access, and non-linear operations) on the specific hardware being considered. For modern deep neural network operations and associated hardware, this constant is generally taken to be 3, \citep{kaplan_scaling_2020, hoffmann_training_2022}, which we adopt for our comparison.

As we will see, the basic unit of computation for both the current implementations of PC and backpropagation algorithms are the same: the VJP. The cost of computing this VJP is bounded in terms of a forward evaluation via the cheap gradient principle. To levy a comparison between algorithms we may therefore report computational complexity with respect to the cost of computing a single forward pass, which we denote with $\mathcal{C}_i$ and $\mathcal{C}_F$, for elementary constituent functions $f_i$, and the overall composite function $F$ respectively. 

Following the notation in \citep{griewank_evaluating_2008} we denote a computational task associated with a function $F$ as $task(F)$, with the corresponding computational and time complexity referred to as $\text{WORK}\{task(F)\}$ and $\text{TIME} \{task(F)\}$. For simplicity, we assume scalar complexity measures for work and time, and relate the time complexity to work complexity for a single elementary function $f_i$ using some positive valued constant $w$, such that $\text{TIME} \{task(f_i)\} = w \text{WORK} \{task(f_i)\} $, note that this relationship may not be true for other high-level tasks on $F$ such as PC inference due to parallel computation, which we will discuss and accommodate for in subsequent sections. Using the cheap gradient principle we can therefore write the following factorisation and bound for the time complexity of backpropagation. 

\begin{align}
    \text{TIME} \{\text{\textit{backprop}}(F)\} &= \text{TIME} \{\text{\textit{forward}}(F)\} + \text{TIME} \{\text{\textit{backward}} (F) \} \\
    &= \sum_i \text{TIME} \{\text{\textit{forward}}(f_i) \} + \sum_i \text{TIME} \{\text{\textit{vjp}}(f_i) \} \\
    &\le w \sum_i \mathcal{C}_i + 2 \mathcal{C}_i =  3 w \sum_i \mathcal{C}_i
\end{align}

\section{Predictive Coding}

\subsection{Variational PC}
We begin by describing the standard formulation of PC as a variational Bayes algorithm, \citep{friston_learning_2003, friston_theory_2005, bogacz_tutorial_2017, buckley_free_2017}. Within this context, the PC algorithm can be interpreted as a method for performing inference (over latent states) and learning (over parameters) for a latent hierarchical generative model. In keeping with much of recent work that renders a comparison with backpropagation we will restrict our focus to PC for static (i.e. non-dynamical/time-series-based) observations and states. We will distinguish this formulation with recent modified formulations such as FPA-PC \citep{millidge_predictive_2020} and Z-IL \citep{song_can_2020, salvatori_predictive_2021} by calling it variational PC (VPC) to emphasise that modified formulations may not necessarily correspond to a variational Bayesian inference and learning algorithm. 

For the sake of enabling our subsequent comparisons of PC with backpropagation, we will reuse the elementary functions $f_i$ first defined in Section \ref{sec:backpropagation} in this section for defining the conditional relationships within our hierarchical model. 

In particular, we will take the outputs of our constituent functions $f_i$ to correspond to means of Gaussian latent random variables, with each Gaussian latent random variable conditioning on its parents. Then, when $f_L$ is a loss function corresponding to the log-likelihood of a particular choice of output distribution $P_L$, (a common problem setting), this results in the following log-joint probability for our probabilistic graphical model:
\begin{align}
\label{eq:log_joint}
    \log p(x_0, ..., x_{L-1}, x_L|\theta_0, ... \theta_L) = \log p(x_L|x_{L-1}, \theta_{L-1}) + ... + \log p(x_1|x_0, \theta_0)
\end{align}
with 
\begin{align}
x_\ell | x_{\ell-1} \sim \left.
\begin{cases}
    \mathcal{N}(f_{\ell-1}(x_{\ell-1}, \theta_{\ell-1}), \Sigma_\ell))
    & \text{for } \ell=1,...,L-1 \\
    \
    P_L(f(x_{L-1, t}, \theta_L))
    & \text{for } \ell=L
\end{cases}
\right\}
\end{align}

Given this model, PC answers the question of how one can learn the parameters $\mathbf{\theta}$ that maximise their model evidence, $\log p(\mathbf{x}_{\text{OBS}}^{(1)}, ... \mathbf{x}_{\text{OBS}}^{(D)}|\boldsymbol{\theta})$, where notational simplicity we have combined all observations ($x_0$ and $x_L$), latent states ($x_1,...,x_{L-1}$), and parameters ($\theta_0,...,\theta_L$) into the vectors $\mathbf{x}_{\text{OBS}}$, $\textbf{x}_{\text{LAT}}$, $\boldsymbol{\theta}$ respectively, and superscript indices denote samples from a data generating (observation) distribution $\mathcal{X}$. 

The standard difficulty with this optimisation of the model parameters to maximise the model evidence rests upon the intractable marginalisation over latent states ($\textbf{x}_{\text{LAT}}$). The variational Bayes solution to this issue rests upon the optimisation of an evidence lower bound (ELBO), equivalently often called the (negative) free energy. This bound relies upon a tractable form for an approximate posterior density over latent states
\begin{align}
   \text{ELBO} = E_{q(\mathbf{x}_{\text{LAT}})}\left[\log p (\mathbf{x}_{\text{OBS}}, \mathbf{x}_{\text{LAT}}|\boldsymbol{\theta})\right] - E_{q(\mathbf{x}_{\text{LAT}})}\left[\log q(\mathbf{x}_{\text{LAT}})\right] 
\end{align}
    
For the approximate posterior density $q(\textbf{x}_{\text{LAT}})$, PC adopts a point-mass (Dirac $\delta$) distribution, either implicitly \citep{buckley_free_2017, bogacz_tutorial_2017, millidge_predictive_2020}, or explicitly \citep{friston_learning_2003, friston_theory_2005}. By denoting the centre of this Dirac $\delta$ delta distribution with $\boldsymbol{\phi}$, the above bound simplifies to the objective
\begin{align}
    \text{ELBO}_{\mathbf{PC}}(\boldsymbol{\phi}, \boldsymbol{\theta}) = \log p (\mathbf{x}_{\text{OBS}}, \boldsymbol{\phi}|\boldsymbol{\theta})
\end{align}

PC optimises this function by first enacting an ascent with respect to the the Dirac $\delta$ parameters $\phi$ corresponding to the modes of the approximate posterior over latent states. Once the maxima for the log joint probability with respect to $\phi$ is obtained, we then update our parameters $\theta$ by computing the gradient with respect to the log joint evaluated at these values of $\phi$.

Depending on one's perspective, this initial step can be seen, variously as: obtaining an MAP (maximum a posteriori) estimate over latent states (due to the maximisation of the log-joint), or as an expectation step within an Expectation-Maximisation scheme \citep{friston_theory_2005}, or as a variational bound tightening step from a variational Bayes perspective. The subsequent maximisation step (optimisation of $\theta$) can then be implemented via a mini-batch or stochastic gradient descent (SGD) procedure allowing one to tractably optimise against a large dataset of observations. The final algorithm can be summarised succinctly as follows:
\begin{enumerate}
    \item Define a (possibly hierarchical) graphical model over latent states ($\mathbf{x}_{\text{LAT}}$) and observations ($\mathbf{x}_{\text{OBS}}$) with parameters $\boldsymbol{\theta}$ \\ (i.e. $\log p (\mathbf{x}_{\text{OBS}}, \mathbf{x}_{\text{LAT}}|\boldsymbol{\theta})$)
    \item For minibatch $\mathbf{x}_{\text{OBS}} \sim \mathcal{D}$, where $\mathcal{D}$ is the data-generating distribution
    \begin{description}
        \item[Inference: ] Obtain MAP estimates ($\mathbf{x}_{\text{MAP}}$) for the latent states by enacting a gradient descent on $-\log p (\mathbf{x}_{\text{OBS}}, \mathbf{x}_{\text{LAT}}|\boldsymbol{\theta})$ 
        \item[Learning: ] Update the parameters $\boldsymbol{\theta}$ using SGD with respect to the negative log joint evaluated at the MAP estimates found at the end of inference: $-\log p (\mathbf{x}_{\text{OBS}}, \mathbf{x}_{\text{MAP}}|\boldsymbol{\theta})$
    \end{description}
\end{enumerate}

For the sake of subsequent discussion it will be useful to look at the exact functional form and computations occurring in  the gradient ascent (inference) procedure outlined in the above algorithm. First, since the conditional log-likelihoods of latent states in the log joint described by Equation (\ref{eq:log_joint}) are Gaussian, these, therefore, take the form of a series of squared precision-weighted prediction errors each corresponding to a constituent function $f_i$, plus the output log-likelihood
\begin{align}
\label{eq:elbo_pc}
    \text{ELBO}_{\text{PC}} &= \log p(x_0, x_1, ..., x_{L-1}, x_L) \span \\ 
    &= f_L(x_L, x_{L-1}) \\
    & \quad + ... \nonumber \\ 
    & \quad + (x_{\ell} - f_{\ell-1}(x_{\ell-1}))^T\Sigma^{-1}_\ell(x_{\ell} - f_{\ell-1}(x_{\ell-1})) \nonumber \\
    & \quad + ...  \nonumber \\
    & \quad + (x_1 - f_0(x_0))^T\Sigma^{-1}_1(x_1 - f_0(x_0)) 
\end{align}
where we have temporarily removed the dependence on parameters $\theta_\ell$ for brevity. 

The gradient of this objective with respect to a particular latent state $x_\ell$ is then simply the precision weighted prediction errors corresponding to predictions of $x_\ell$ from its parents as well as precision weighted prediction errors corresponding to predictions $x_\ell$ makes of its child nodes $x_{\ell+1}$. The gradient descent (inference) procedure can therefore be described by the following discretised gradient flow, assuming a Gaussian output log-likelihood for $f_L$ (i.e. Euclidean output loss): 
\begin{align}
    x_{\ell, t+1} 
    &= x_{\ell, t} - \gamma \left.\frac{\partial F}{\partial x_{\ell}}\right|_{x_{\ell, t}} \\ 
    &= x_{\ell, t} - \gamma \left[\Sigma^{-1}_{\ell}(x_{\ell, t} - f_{\ell-1}(x_{\ell-1,t})) - \left.\frac{\partial f_{\ell}}{\partial x_{\ell}}\right|_{x_{\ell, t}}^T \Sigma^{-1}_{\ell+1}(x_{\ell+1, t} - f(x_{\ell, t})) \right]
    \label{eq:bpc_dynamics}
    \intertext{where $\gamma$ is some inference step size. Under the assumption of identity variances, this equation simplifies to}
    &= x_{\ell, t} - \gamma \left[(x_{\ell, t} - f_{\ell-1}(x_{\ell-1,t})) - \left.\frac{\partial f_{\ell}}{\partial x_{\ell}}\right|_{x_{\ell, t}}^T (x_{\ell+1, t} - f(x_{\ell, t}))  \right]
    \label{eq:bpc_dynamics_2} \\
    \intertext{Written more generically in terms of errors, we can write this to include arbitrary definitions of output loss log-likelihood}
    &= x_{\ell, t} - \gamma \left[e_{\ell, t} - \left.\frac{\partial f_{\ell}}{\partial x_{\ell}}\right|_{x_{\ell, t}}^T e_{\ell+1, t}  \right] \\
    \intertext{with}
     e_{\ell, t} &= \left.
    \begin{cases}
    (x_{\ell, t} - f_{\ell-1}(x_{\ell-1,t}))
    & \text{for } \ell=1,...,L-1 \\
    \
    \frac{\partial f_L}{\partial f_{L-1}}^T
    & \text{for } \ell=L
    \end{cases}
\right\}
\end{align}
where we abuse notation and use $f_{L-1}$ to denote both the function, and its output evaluated at the current value of $x_{L-1}$, i.e. $f_{L-1}(x_{L-1, t})$. We do this to distinguish it from $\mu_{L}$, and the associated Jacobian in Equation (\ref{eq:recursive_backprop_equations}), which denotes the pre-loss output prediction computed using the feed-forward value $\mu_{L-1}$, i.e. $f_{L-1}(\mu_{L-1})$. 
That is to say, at every inference step, the latent states are updated such that they act to minimise the error corresponding to the current prediction of their latent states, and the error corresponding to the current prediction of their children's latent states. 

The model parameters are then updated (via SGD) with the derivative of the log joint evaluated at these converged MAP values. This gradient is the product of the error associated with the conditional distribution it parameterises and the Jacobian of that function with respect to $\theta$. Let $t_c$ be the time at which the variational modes have converged. Then, the gradient of the log joint with respect to the model parameters is given by
\begin{align}
    \label{eq:bpc_theta_updates}
    \frac{\partial F}{\partial \theta_{\ell}} = \frac{\partial f_{\ell}}{\partial \theta_{\ell}}^T e_{\ell+1, t_c}
\end{align}

Equations (\ref{eq:bpc_dynamics}-\ref{eq:bpc_dynamics_2}) describe the dynamics of the standard PC formulation. Under these dynamics, inference over latent states can be considered MAP inference on a corresponding probabilistic model, and learning can be considered as the maximisation of an ELBO with respect to its model parameters. In particular, this bound corresponds to a point-mass variational distribution assumed over our latent states.  

\paragraph{Inverted Configurations: } This particular formulation of PC is inverted relative to the standard PC formulation \citep{rao_predictive_1999, friston_learning_2003, friston_hierarchical_2008, buckley_free_2017} present in cognitive science, and indeed most generative modelling schemes, in the sense that observations parameterise, and are thus hierarchically higher than, latents such as classification labels. Inverting PC in this way, to render a comparison with backpropagation, has non-trivial implications on neurobiological plausibility. One of the strengths of standard PC formulations has historically been in the notion that latents hierarchically higher in the cortex have a non-linearly mixing and modulatory effect ($f(x_\ell, \theta_\ell))$) that is congruent with what we know about cortical anatomy; i.e. that top-down backward connections are generally more bifurcating and modulatory, while bottom-up influences, or forward connections, are more driving ($-e_\ell$) \citep{friston_theory_2005, friston_hierarchical_2008, bastos_canonical_2012, shipp_neural_2016, markov_anatomy_2014}. 

It is interesting to note that supervised generative (i.e non-inverted) formulations of standard PC have thus far struggled to produce competitive or close to competitive results for complex classification tasks. For example, traditional generative PC models trained to classify and generate MNIST digits are consistently unable to surpass a test set accuracy of ~80\%, both from our experience and observed here \citep{kinghorn_preventing_2022}. The authors are not aware of any example of a traditional (generative) PC succeeding in obtaining competitive classification performance on standard machine learning tasks. Though the possible reasons for this are outside the scope of this paper, it is likely a consequence of either the Gaussian-based supervision signal generally used in existing attempts, or the limitations of the Dirac $\delta$ (point-mass) approximate posterior, for training generative models, as highlighted in \citep{zahid_curvature-sensitive_2023}.

\subsection{Memory locality, weight transport and weight symmetry}

\begin{figure}[h!]
\centering
\includegraphics[width=0.9\linewidth]{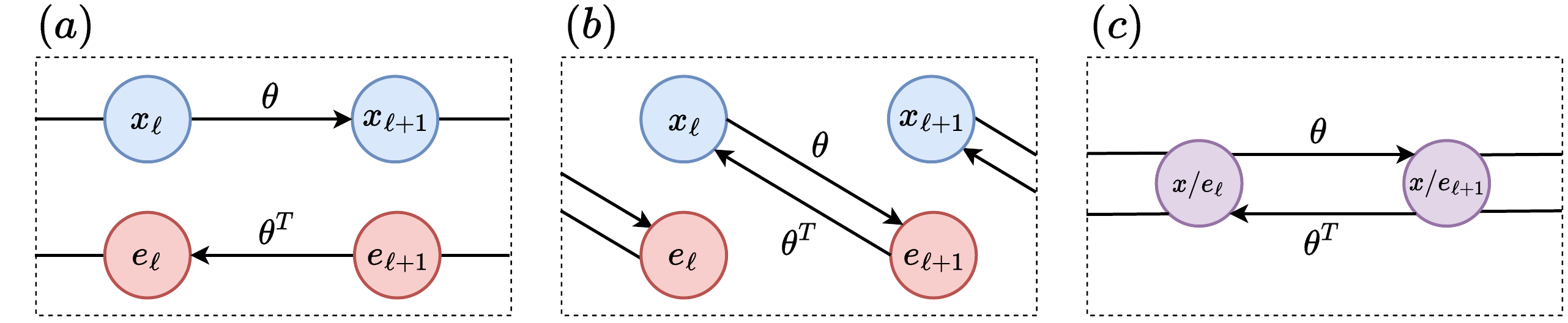}
\caption{(a) Weight transport problem in a standard proposal implementation for backpropagation. (b) Weight symmetry issue in standard PC formulations. Note we exclude additional connectivity necessary for PC, but not relevant to the weight transport problem, for the sake of clarity. (c) An example alternative proposal aimed at solving the weight transport issue for backpropagation \citep{amit_deep_2019}} 
\label{fig:weight_transport_vs_symmetry}
\end{figure}

Local computation within a neurobiological context generally refers to local plasticity, i.e. synaptic weight updates that are determined by the activity of the neurons they connect, that is to say pre-synaptic and post-synaptic neuronal activity, and synaptic locality. This notion of locality is often used in the context of assessing learning algorithms for their plausibility of implementation in the brain. Historically, a key criticism in the biological plausibility of backpropagation has been due to the presence of this type of non-locality, where it has been called the weight transport problem \citep{grossberg_competitive_1987, crick_recent_1989, zipser_neurobiological_1993}.

Under the problem setting described above Equations (\ref{eq:forward}), the weight transport problem can be summarised as the issue of a particular weight or parameter matrix (e.g. $\theta_\ell$) involved in the computation of feed-forward activation $\mu_{\ell+1}$, being transported for use in the computation of errors $e_\ell$. Specifically, if one considers a constituent function $f_\ell$ consisting of an affine transformation given by the weight matrix $\theta_\ell$ and a non-linearity $g(\cdot)$, such that
\begin{align}
    \label{eq:constituent_function_affine}
    \mu_{\ell+1} &= f_\ell(\mu_\ell, \theta_\ell) = g(\theta_\ell \mu_\ell) \\
    \intertext{The computation of $e_\ell$ is then}
    e_\ell &= \frac{\partial \mu_{\ell+1}}{\partial \mu_{\ell}}^T e_{\ell+1}
 = g'(\theta_\ell \mu_\ell) \theta_\ell^T e_{\ell+1} 
\end{align}
where we require reusing the forward weights due to the transpose weight term $\theta_\ell^T$. 

With respect to backpropagation there have been a number of well-known attempts at resolving the implausibility of the weight transport issue. These include the influential work \citep{lillicrap_random_2016}, where it was shown that computing the error via a fixed random matrix $B$ replacing $\theta_\ell$ could nonetheless facilitate learning without the requirement that the forward parameters be reused. However, the use of random feedback weights only remained performant for shallow networks, failing for deeper models. Subsequent and more recent work has demonstrated that training these feedback weights separately can improve performance for deeper models \citep{amit_deep_2019, akrout_deep_2019}. 

While these various models resolve the weight transport issue in one way or another, they frequently require or induce additional assumptions about the neurobiological machinery required for their implementation. The work of \citep{zipser_neurobiological_1993} for example required initialising feedforward and feedback matrices identically, which was criticised and resolved using weight decay by \citep{kolen_backpropagation_1994}. Both \citep{kolen_backpropagation_1994} and \citep{zipser_neurobiological_1993} however required transmitting the synaptic weight changes between separate networks, which has itself also been argued as implausible \citep{akrout_deep_2019}. \citep{amit_deep_2019} required strict and regimented sequencing of computation, with feed-forward neurons subsequently also acting as error neurons, which can be seen diagramatically in Figure \ref{fig:weight_transport_vs_symmetry}. 

Lastly, for the sake of further comparison it, both \citep{lillicrap_random_2016} and \citep{akrout_deep_2019} require interactions between the activations within forward and backward pathways, usually implemented as three-factor style learning rules, which induce a complexity beyond simple two-factor Hebbian learning rules. Though, we note, this may not be particularly problematic from the standpoint of neurobiological plausibility given a growing body of literature presenting empirical evidence, or mechanistic proposals for three-factor learning rules \citep{sjostrom_cooperative_2006, pawlak_timing_2010, urbanczik_learning_2014}.

A very similar but slightly different issue to weight transport exists within PC due to the presence of the Jacobian term ($\frac{\partial f_\ell}{\partial x_\ell}$) in Equations \ref{eq:bpc_dynamics} and \ref{eq:bpc_dynamics_2}. For the same definition of $f_\ell$ (affine + non-linearity), this results in the following inference equations for standard PC
\begin{align}
    x_{\ell, t+1} &= x_{\ell, t} - \gamma \left[e_{\ell, t} - \left.\frac{\partial f_{\ell}}{\partial x_{\ell}}\right|_{x_{\ell, t}}^T e_{\ell+1, t} \right] \\
    &= x_{\ell, t} - \gamma \left[e_{\ell, t} - g'(\theta_\ell \mu_\ell) \theta_\ell^T e_{\ell+1, t}  \right] 
    \intertext{with $\theta_\ell$ also being used in the computation of $e_{\ell+1}$}
    e_{\ell+1, t} &= x_{\ell+1, t} - g(\theta_\ell x_\ell)
\end{align}

Thus, the (negative) identical of the weights that mediate the influence of the subsequent error units ($e_{\ell+1}$) on the previous latent states ($x_\ell$), are also required for the reciprocal connection from latent states to error units. See Figure \ref{fig:weight_transport_vs_symmetry} for a depiction of this.  This property was identified in some of the earliest narratives on PC \citep{friston_theory_2005, friston_hierarchical_2008} albeit from the perspective of an advantage of PC's neurobiological plausibility due to the reported prevalence of reciprocal connections in the brain \citep{felleman_distributed_1991}. Prevalence of reciprocal connections is not sufficient however, as we, at least naively, require exact equivalence in strengths also. 

We distinguish this issue from that of weight transport by calling it the weight symmetry problem as it requires weight equivalence for forward and reciprocal synapses between the same two pairs of neurons, which is arguably a more addressable issue from the perspective of neurobiological plausibility over the required weight transport in backpropagation. The reason for this is that plasticity is generally considered to be some simple function of pre-synaptic and post-synaptic activity, given that both these activities are accessible for forward and reciprocal synapses it is not implausible that forward and backward weights could come into parity. In particular, given a simple 2-factor Hebbian learning rule, one would expect, at least for simple Hebbian plasticity, that both forward and reciprocal connections could experience the same weight modifications (the product of pre-synaptic and post-synaptic activity). Further incorporating weight decay, would result in weights that eventually synchronise, as demonstrated for the weight transport case by \citep{kolen_backpropagation_1994}. This approach, without decay, has indeed been demonstrated with minimal performance degradation on standard PC networks \citep{millidge_relaxing_2020}.  

We will next discuss a number of recent modified formulations of PC which have been presented as neurmorphic alternatives to backpropagation.

\subsection{Backpropagation as PC in the infinite variance limit}

\begin{figure}[h!]
\centering
\includegraphics[width=0.6\linewidth]{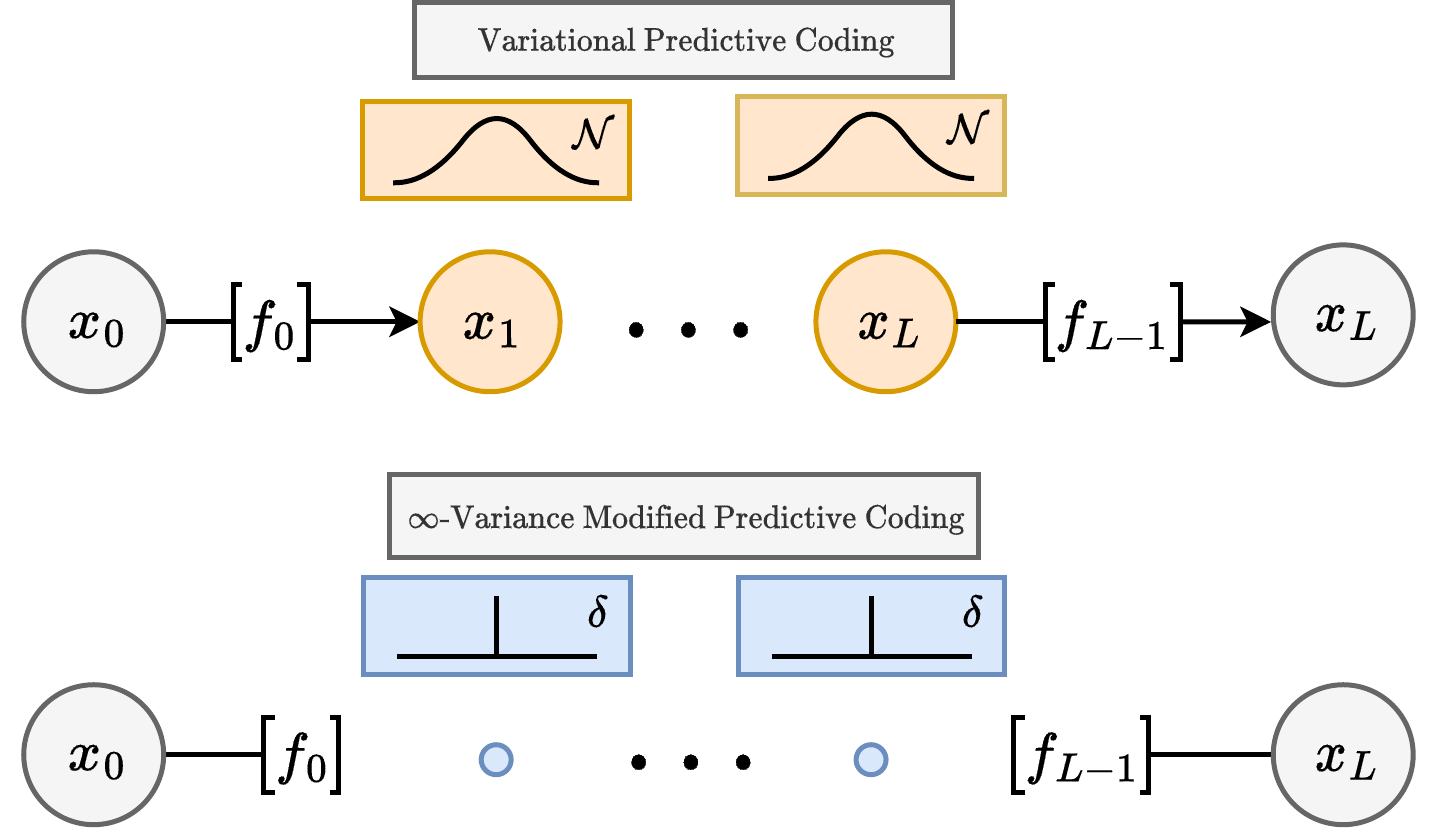}
\caption{Underlying probabilistic model assumed by standard, i.e. variational, PC (VPC) and recent modified forms of PC. Note VPC incorporates probabilistic latent states whereas FPA-PC assumes deterministic states and observations only. Grey circles correspond to observed variables, blue circles correspond to deterministic intermediate states, and orange circles correspond to probabilistic latent states.}
\label{fig:underlying_pgm}
\end{figure}

The link between PC and backpropagation was first presented in \citep{whittington_approximation_2017}, where it was shown that the resultant weight updates under the PC algorithm outlined above would equal that for the equivalent backpropagation graph given the following specific restrictive set of circumstances: 

\paragraph{Requirement 1} 
\begin{itemize}
    \item When the variational modes or latent states ($x_1, ..., x_{L-1}$) are initialised to the feed-forward values of the corresponding backpropagation-based computational chain ($\mu_1, ..., \mu_{L-1}$), \textit{and}: 
\end{itemize}

\paragraph{Requirement 2}
\begin{itemize}
    \item \textbf{Case 1} When the feed-forward prediction results in zero output loss (i.e. zero negative log-likelihood loss) 
    \item \textbf{Case 2} When the output variance ($\Sigma_L$) is set sufficiently higher than the remaining variances ($\Sigma_{\ell<L}$), such that the ratio of $\left(\frac{\Sigma_L}{\Sigma_{\ell}}, \forall \ell < L \right)$, goes to infinity \textit{and} the learning rate is scaled appropriately. 
\end{itemize}

Both these cases correspond to the same requirement, namely that the values of the latent states, $x_1,...,x_{L-1}$, remain equal, or close, to their feed-forward values at the end of the inference procedure. Said another way, this can be seen as the requirement that the MAP values for the latent states while observing both input observations (e.g. images) and output observations (e.g. classification labels), are equal to the MAP values when conditioning on input observations alone. Another more intuitive framework for understanding this requirement is to note that the inference and learning procedure that results under these set of circumstances is equivalent to doing inference and learning for a probabilistic graphical model in which the intermediate latent states have 0 conditional uncertainty, or variance, and are thus deterministic outputs of the input observations. 

Given this framework, learning devolves into enacting maximimum likelihood learning for a probabilistic graphical model in which inputs $x_0$ parameterise an output likelihood via the deep "total" network function $F = f_L \circ f_{L-1} ... \circ f_{1} \circ f_{0}$, with observations $x_L$. This is equivalent to enacting maximum likelihood training with backpropagation when the output loss function $f_L$ (see Equations \ref{eq:forward}) corresponds to a valid log-likelihood. See Figure \ref{fig:underlying_pgm} for a depiction of the probabilistic graphical model assumed by this approach versus standard VPC. 

To see this, observe that this is trivially true for \textbf{Case 1}, where the weight updates from the aforementioned deterministic latent maximum likelihood objective and the PC objective with probabilistic latent states are the same, namely zero, due to the zero error. 

More interestingly, this can also be seen for \textbf{Case 2}, if one considers that scaling up $\Sigma_L$ and rescaling all learning rates by its inverse, is equivalent to scaling \textit{down} the intermediate variances $\Sigma_{\ell}, \forall \ell < L$, while keeping the output variances fixed, and keeping the learning rate the same. This is arguably a more natural perspective to view this procedure as it allows us to understand what is happening to the underlying probabilistic graphical model that we are training as we do this: namely that we are removing uncertainty over latent states, and thus going from a variational Bayes (free energy minimising) learning procedure to a maximum likelihood learning procedure. To demonstrate this, we can look at what happens to the inference and learning equations as we conduct the \textbf{Case 2} scaling procedure from \citep{whittington_approximation_2017}: 

Consider that we have scaled the label variances $\Sigma_L$ by a large constant factor $k$, and subsequently scaled the parameter learning rate ($\alpha$) by $\frac{1}{k}$ to compensate, this results in the following modified inference and learning equations (previously Equations \ref{eq:bpc_dynamics} and \ref{eq:bpc_theta_updates}), again assuming a Gaussian output log likelihood (squared Euclidean output loss)
\begin{align}
    x_{\ell, t+1} = x_{\ell, t} &-
    \gamma
    \left.
    \begin{cases}
     \begin{multlined}
        \Sigma^{-1}_{\ell}(x_{\ell, t} - f_{\ell-1}(x_{\ell-1,t})) 
            \\ - \left.\frac{\partial f_{\ell}}{\partial x_{\ell}}\right|_{x_{\ell, t}}^T \Sigma^{-1}_{\ell+1}(x_{\ell+1} - f(x_{\ell, t}))
     \end{multlined} 
    & \text{for } \ell=1,...,L-2 \\
    \nonumber \\
    \
    \begin{multlined}
        \Sigma^{-1}_{L-1}(x_{L-1, t} - f_{L-2}(x_{L-2,t})) 
            \\ - \left.\frac{\partial f_{L-1}}{\partial x_{L-1}}\right|_{x_{L-1, t}}^T (\frac{1}{k} \cdot \Sigma^{-1}_{L})(x_{L} - f(x_{L-1, t})) 
    \end{multlined}
    & \text{for } \ell=L-1
    \end{cases}
\right\}
\end{align}
which may be rewritten as
\begin{align}
    x_{\ell, t+1} = x_{\ell, t} &-
    \gamma'
    \left.
    \begin{cases}
    \begin{multlined}
        (k \cdot \Sigma^{-1}_{\ell})(x_{\ell, t} - f_{\ell-1}(x_{\ell-1,t})) 
            \\ - \left.\frac{\partial f_{\ell}}{\partial x_{\ell}}\right|_{x_{\ell, t}}^T (k \cdot \Sigma^{-1}_{\ell+1})(x_{\ell+1} - f(x_{\ell, t})) 
    \end{multlined}
    & \text{for } \ell=1,...,L-2 \\
    \nonumber \\
    \
    \begin{multlined}
        (k \cdot \Sigma^{-1}_{L-1})(x_{L-1, t} - f_{L-2}(x_{L-2,t})) 
            \\ -  \left.\frac{\partial f_{L-1}}{\partial x_{L-1}}\right|_{x_{L-1, t}}^T (1 \cdot \Sigma^{-1}_{L})(x_{L} - f(x_{L-1, t}))
    \end{multlined}
    & \text{for } \ell=L-1
    \end{cases}
\right\}
\end{align}
where, for inference, $\frac{1}{k}$ factor has been absorbed by a new step size $\gamma'$ and, thus, given Euler integration does not diverge, will result in the same converged errors at the end of inference. 
The updates to our parameters $\theta_\ell$ are then
\begin{align}
\Delta \theta_\ell &= \frac{1}{k} \alpha 
\left.
    \begin{cases}
     \frac{\partial f_{\ell}}{\partial \theta_{\ell}}^T  \Sigma^{-1}_{\ell+1}(x_{\ell+1} - f(x_{\ell, t})) 
    & \text{for } \ell=1,...,L-2 \\
    \nonumber \\
    \
    \frac{\partial f_{\ell}}{\partial \theta_{\ell}}^T
    (\frac{1}{k} \cdot \Sigma^{-1}_{L})(x_{L} - f(x_{L-1, t}))
    & \text{for } \ell=L-1
    \end{cases} 
\right \} \\
\intertext{which may be rewritten as}
\Delta \theta_\ell &= \alpha 
\left.
    \begin{cases}
     \frac{\partial f_{\ell}}{\partial \theta_{\ell}}^T (\frac{1}{k} \cdot \Sigma^{-1}_{\ell+1})(x_{\ell+1} - f(x_{\ell, t})) 
    & \text{for } \ell=1,...,L-2 \\
    \nonumber \\
    \
    \frac{\partial f_{\ell}}{\partial \theta_{\ell}}^T
    (1 \cdot \Sigma^{-1}_{L})(x_{L} - f(x_{L-1, t}))
    & \text{for } \ell=L-1
    \end{cases}
\right \}
\end{align}

Therefore, the modification in \textbf{Case 2} corresponds to downscaling the variances of the intermediate latent states towards 0 while keeping the output likelihood variance fixed. Obtaining exact equivalence to the equivalent backpropagation-based learning update for this non-latent graph would require taking k to infinity, or more practically for getting approximate equivalence, some reasonably large value; \citep{whittington_approximation_2017} uses k=100 for a small network. But this scaling clearly runs the risk of being numerically unstable and is also biologically unrealistic.    

\subsection{The Fixed-Prediction Assumption}
\label{sec:fpa_pc}

An alternative to the scaling procedure mentioned above, dubbed the \textit{"fixed-prediction assumption"} (FPA) was presented in \citep{millidge_predictive_2020}, wherein the top-down predictions $f(x_\ell)$ and Jacobians $\left(\frac{\partial f_\ell}{\partial x_\ell}\right)$ corresponding to each latents were fixed to their feed-forward predictions $\mu_{\ell+1}$ and $\left(\frac{\partial f_\ell}{\partial x_\ell}|_{\mu_{\ell}}\right)$ respectively, throughout the inference procedure. 

The FPA modification of PC fixes the following terms in the discretised gradient flow to specific values corresponding to their feed-forward initialisation. That is to say
\begin{alignat}{2}
    f_{\ell-1}(x_{\ell-1, t}) &= f_{\ell-1}(x_{\ell, 0}) &&= f_{\ell-1}(\mu_{\ell-1}) \label{eq:fpa_fixing_first} \\[0.5em]
    f_{\ell}(x_{\ell}, t) &= f_{\ell}(x_{\ell, 0}) &&= f_{\ell}(\mu_{\ell}) \\[0.5em]
    \left.\frac{\partial f_{\ell}}{\partial x_{\ell}}\right|_{x_{\ell, t}} &= \left.\frac{\partial f_{\ell}}{\partial x_{\ell}}\right|_{x_{\ell, 0}} &&= \left.\frac{\partial f_{\ell}}{\partial x_{\ell}}\right|_{\mu_{\ell}} \label{eq:fpa_fixing_last}
\end{alignat}
where we have assumed \textbf{Requirement 1}, i.e. that the latent states are initialised at the feed-forward values of the corresponding computational chain. We refer to the inference updating equation corresponding to these FPA modifications as $u(x_{\ell, t}, x_{\ell+1, t})$, for a particular latent $x_{\ell}$, to show that it is dependant on the instantaneous value of only the $x_{\ell, t}$ and it's child variable $x_{\ell+1, t}$, and not the parent latent states $x_{\ell-1}$. This decoupling of the dynamics of each latent state from the values of its parents means that the trajectory of the latent states is no longer guaranteed to follow the gradient of the log-joint (or the free-energy as defined for standard PC (Equation \ref{eq:elbo_pc})), and therefore these dynamics are also not guaranteed to find the MAP estimate of the latent states for the model defined in standard VPC,
\begin{align}
\label{eq:fpa_pc_update_equations_1}
 x_{\ell, t+1} = x_{\ell, t} - \gamma \, u(x_{\ell, t}, x_{\ell+1, t})
\end{align}
with
\begin{align}
\label{eq:fpa_pc_update_equations_2}
 u(x_{\ell}, x_{\ell+1}) = \left.
    \begin{cases}
    (x_\ell - \mu_\ell) - \frac{\partial f_{\ell}}{\partial x_{\ell}}^T (x_{\ell+1} - \mu_{\ell+1})
    & \text{for } \ell=1,...,L-2 \\
    \
    (x_{L-1} - \mu_{L-1}) - \left.\frac{\partial f_{L-1}}{\partial x_{L-1}}\right|_{\mu_{L-1}}^T  \frac{\partial f_L}{\partial \mu_{L}}^T
    & \text{for } \ell=L-1
    \end{cases}
\right\}
\end{align}

It was suggested in \citep{millidge_predictive_2020} that the aforementioned fixed-prediction modifications (Equations \ref{eq:fpa_fixing_first} and \ref{eq:fpa_fixing_last}) were equivalent to the infinite variance limit of \textbf{Case 2} from \citep{whittington_approximation_2017}. This is due to the fact that in the limit of infinite output variance, the predictions of the latent states change very little from their feed-forward initialisations by the end of inference. We argue however that this equivalence does not hold due to the fact that in the infinite or high-variance limit of \textbf{Case 2} the latent states $x_\ell$ also remain fixed, or arbitrarily close, to their feed-forward values, and not just the predictions $f_\ell(x_\ell)$. This behaviour is necessary for the inference to accurately correspond to inference over the latent intermediate states $x_\ell$. Which, in the limit of the zero-variance (deterministic) intermediate latent states assumed by the equivalent backpropagation model, become equal to the feed-forward values $\mu_\ell$.

This is not the case for the FPA modification of PC, where the latent states are unconstrained and the prediction error is necessarily not zero except in the trivial case of a zero magnitude weight update. 

This suggests that FPA-PC inference corresponds neither to the probabilistic graphical model assumed by standard VPC, nor that of the scaled variance (i.e. deterministic latent states) modification from \textbf{Case 2} above. This leaves us with the question of how one can interpret FPA-PC inference if not in terms of inference over a probabilistic graphical model as with standard VPC. One particularly simple and elegant answer to this is that one can interpret the FPA-PC modified inference equations as directly implementing the recursive backpropagation equations via the steady state of a set of ordinary differential equations. FPA-PC inference can then be interpreted as enacting an Euler integration scheme until one reaches this steady state. We derive this perspective in the following section. 

\subsubsection{Deriving FPA-PC as a steady-state implementation of backpropagation}

We can make the relationship between FPA-PC and backpropagation even more clear by showing that the FPA-PC equations can be interpreted as a direct implementation of backpropagation via differential equations. To illustrate this, we work our way backwards from backpropagation, showing how a direct implementation of its recursion relationship results in the FPA-PC inference equations. 

The key quantities which backpropagation requires computing are the errors $e_i$ via the recursive expressions (Equations~\ref{eq:recursive_backprop_equations}). We can trivially convert these expressions such that they correspond to the steady state of a set of simple differential equations
\begin{align}
\dot{e}_\ell = \left.
- \begin{cases}
    e_{\ell} - \frac{\partial \mu_{\ell+1}}{\partial \mu_{\ell}}^T e_{\ell+1}
    & \text{for } \ell=1,\dots,L-1 \\
    e_{\ell} - \frac{\partial f_L}{\partial \mu_L}^T
    & \text{for } \ell=L
\end{cases}
\right\}
\end{align}
such that when $\dot{e}_\ell = 0$, we obtain the required recursive relationships (Equation \ref{eq:recursive_backprop_equations}). 
Note that $\frac{\partial \mu_{\ell+1}}{\partial \mu_{\ell}}$ and $\frac{\partial E}{\partial\mathbf{\mu_L}}$ are fixed Jacobians and require being evaluated at the feed-forward values $\mu_\ell$. Also note, the outer negative sign has been added to make the relationship with FPA-PC more clear, and this does not impact the value of $e_\ell$ at the steady state. 

We can then rewrite $e_\ell$ in terms of some variable $x_\ell$ minus the constant feed-forward values $\mu_\ell$ so that we have: $e_{\ell} = x_\ell - \mu_\ell$, allowing us to write the above dynamic equations in terms of a variable $x_\ell$ as
\begin{align}
\dot{x}_\ell = \left.
- \begin{cases}
    (x_\ell - \mu_\ell) - \frac{\partial \mu_{\ell+1}}{\partial \mu_{\ell}} ^T (x_{\ell+1} - \mu_{\ell+1})
    & \text{for } \ell=1,\dots,L-1 \\
    (x_\ell - \mu_\ell) - \frac{\partial f_L}{\partial \mu_L}^T
    & \text{for } \ell=L
\end{cases}
\right\}
\end{align}
If we instead assign $e_L$ outright, as done by FPA-PC, and keep the remaining errors dynamic, we obtain
\begin{align}
\dot{x}_\ell = \left.
- \begin{cases}
    (x_\ell - \mu_\ell) - \frac{\partial \mu_{\ell+1}}{\partial \mu_{\ell}} ^T (x_{\ell+1} - \mu_{\ell+1})
    & \text{for } \ell=1,\dots,L-2 \\
    (x_{L-1} - \mu_{L-1}) - \frac{\partial \mu_L}{\partial \mu_{L-1}} ^T\frac{\partial f_L}{\partial \mu_L}^T
    & \text{for } \ell=L-1
\end{cases}
\right\}
\end{align}
which are precisely equal to the update equations (\ref{eq:fpa_pc_update_equations_1}-\ref{eq:fpa_pc_update_equations_2}) of FPA-PC.

The practical consequences of this modified version of PC were first  questioned in \citep{rosenbaum_relationship_2022}, who showed that in case of step size equal to 1, FPA-PC is not just functionally but algorithmically equivalent to backpropagation, suggesting at least in the specific case of inference step size equal to 1, there may be no benefit of FPA-PC over backpropagation. We extend this work and investigate the properties of \citep{millidge_predictive_2020}, as well as other recent modifications \citep{song_can_2020, salvatori_predictive_2021} further, showing that these modifications have provable worse computational and time complexity for any step size. 

\subsubsection{Computational and Time Complexity}

To investigate the computational and time complexity of the FPA-PC and other modified PC algorithms we begin by proving that under the aforementioned dynamics, the number of inference steps taken for an error to propagate from an output node in a computational chain to an intermediate node is lower-bounded by its distance to the output node. See Theorem \ref{theorem:1} and Corollary \ref{corollary:1}. 

\newtheorem{theorem}{Theorem}
\newtheorem{lemma}[theorem]{Lemma}
\newtheorem{corollary}{Corollary}[theorem]
\begin{theorem}
\label{theorem:1}
For a particular node ($x_\ell$) in a computational chain that has non-zero gradient w.r.t an output loss $\mathcal{L}$, initialised to feed-forward values of the network, evolving under standard PC or FPA-PC dynamics, the error for that node ($e_\ell$) will first become non-zero at time $t = L - \ell$. Where L is the length of the chain, and $\ell$ is the index of that node within the chain.
\end{theorem}

\begin{proof}
Let the variational modes $x$ be initialised to feed-forward values of the network, (i.e. \textbf{Requirement 1}). That is, let $x_\ell = \mu_\ell$ for $\ell \in [1,...,L-1]$. \\

Let Equations \ref{eq:fpa_pc_update_equations_1} and \ref{eq:fpa_pc_update_equations_2} describe the dynamics of our latent states under FPA-PC dynamics, and Equations \ref{eq:bpc_dynamics} and \ref{eq:bpc_dynamics_2} describe the dynamics under standard PC.

\begin{quote}
\begin{description}
\item[1] Because we are ignoring the trivial case of a zero gradient with respect to the output loss, we are guaranteed that the Jacobians $\frac{\partial f_\ell}{\partial x_\ell}$ are all non-zero as $\frac{\partial \mathcal{L}}{\partial x_{\ell}} = \frac{\partial f_\ell}{\partial x_\ell}...\frac{\partial f_{L-1}}{\partial x_{L-1}}\frac{\partial \mathcal{L}}{\partial f_{\ell-1}} \ne 0$.
\item[2] At time $t=0$, the dynamics $u(x_{\ell}, x_{\ell+1})$ of all nodes are $0$, except for $x_{L-1}$, by the definition of the update equations and \textbf{Requirement 1}. 
\item[3] At time $t=1$, the nodes $x_{L-1}$ therefore update proportionally to $u(x_L)$
\item[4] At any particular time step $t=t+1$, if the dynamics $u(x_{\ell}, x_{\ell+1})$ associated with a node $x_\ell$ were $0$ in time step $t$, then they will become non-zero at $t+1$, if and only if, a change has occurred in $x_{\ell+1}$ from time step $t$ to $t+1$. (By the definition of the FPA-PC (\ref{eq:fpa_pc_update_equations_1} and \ref{eq:fpa_pc_update_equations_2}) or standard PC (\ref{eq:bpc_dynamics} and \ref{eq:bpc_dynamics_2}) update equations, and statement 1)
\item[5] Thus, via induction, the first non-zero change to occur for an arbitrary variational mode $x_\ell$ at point $\ell$ in the computational chain will occur at a time $t = L - \ell$, i.e. its distance from the output with regards to the number of intermediary nodes.
\end{description}
\end{quote}
\end{proof}

\begin{corollary}
\label{corollary:1}
The convergence time for errors associated with an arbitrary node in a computational chain is lower-bounded by its distance, in nodes, to the output node. 
\end{corollary}
\begin{proof} 
~\begin{quote}
\begin{description}
\item[1] The fixed convergence point for the dynamical equations of FPA-PC is guaranteed to converge to an error equal to the gradient with regards to the output loss.  
\item[2] For a node, with non-zero gradient with regards to an output loss, to converge, the converged error for that node must therefore be non-zero.
\item[3] Under theorem 1, the time taken for the error to first become non-zero is equal to its distance to the output node $x_L$, which is $t = L - \ell$ for a computational chain of length $L$.
\end{description}
\end{quote}
\end{proof}

In words, this proof describes how, despite the ostensibly parallel computation occurring in FPA-PC, convergence nonetheless requires sufficient time steps for error information to propagate from the output node backward. This is necessary as information is still nonetheless only transmitted via adjacent nodes, as with backpropagation. Thus, inference must proceed for a minimum number of inference steps equal to this distance before convergence can occur. 

We can now ask what the time complexity of a single inference step is under FPA-PC dynamics, and thus the time complexity of inference overall for FPA-PC, as well as how this compares to backpropagation. Note that, for this comparison, we exclude the final VJP calculation for gradients associated with a particular $\theta$, for both FPA-PC and backpropagation, as this has the same computational cost for both algorithms and only occurs once, at the end of inference and error propagation respectively.   

When considering the time complexity of FPA-PC inference we will assume the error computations ($e_{\ell} = x_{\ell} - \mu_{\ell}$) contribute negligible latency to the computation at every inference step. This is a generally reasonable assumption for the intermediate errors, given that these errors can be computed entirely in parallel and thus have the latency of a single subtraction operation - which is also generally amongst the lowest latency arithmetic instructions available on modern hardware (See \citep{wong_demystifying_2010, fog_instruction_2011}). 
\begin{align}
    \text{TIME} \{\text{\textit{FPA-PC}}(F)\} &= \text{TIME} \{\text{\textit{forward}}(F)\} + \text{TIME} \{\text{\textit{FPA-Inf}} (F) \} \label{eq:fpa_time_complexity_first}
\end{align}
We assume each inference step is maximally parallel, such that VJP computations across all nodes occur simultaneously. However, since each inference step must wait for the slowest (VJP) computation to complete, we can take the minimum time complexity for each inference step to be equal to that of the slowest VJP computation: (Note, a similar result also follows if we assume the VJP computation for each function $f_i$ takes roughly the same amount of time)
\begin{align}
    \text{TIME} \{\text{\textit{FPA-PC}}(F)\} &= \sum_{i=0}^{L} \text{TIME} \{\text{\textit{forward}}(f_i) \} + \sum_{t=0}^{t_c} \text{TIME} \{\text{\textit{vjp}}(f_{\text{slow}}) \} \\
    &\ge \sum_{i=0}^{L} \text{TIME} \{\text{\textit{forward}}(f_i) \} + \sum_{t=0}^{L} \text{TIME} \{\text{\textit{vjp}}(f_{\text{slow}}) \} \\
    &\ge w \sum_{i=0}^{L} \mathcal{C}_i + w \sum_{i=0}^{L} 2 \mathcal{C}_{\text{max}} \\ 
    &\ge w \sum_{i=0}^{L} \mathcal{C}_i + w \sum_{i=0}^{L} 2 \mathcal{C}_i \\
    &\ge \text{TIME}\{\text{\textit{backprop}}(F)\}
    \label{eq:fpa_time_complexity_last}
\end{align}
where $C_{\text{max}} = \max \{C_i :  i = 1, ..., L\} = \text{WORK} \{\text{\textit{vjp}}(f_{\text{slow}}) \}$.
Thus, the time complexity for FPA-PC inference is provably greater than, or equal to, that of backpropagation for an equivalent computational chain. The bound we present here is a close to best-case scenario for FPA-PC for a number of reasons: 

\begin{enumerate}
    \item We have assumed error computation ($e_{\ell} = x_{\ell} - \mu_{\ell}$) contributes negligible latency to each inference step.
    \item Convergence will generally not occur within exactly $t_c = L$ steps except in particular circumstances (e.g. inference step size equal to 1), and so in practice, FPA-PC may have a significantly higher time complexity (scaling with inference length) than backpropagation. 
    \item Bottlenecks (slow VJP computations) will impact FPA inference significantly more as each inference step will be as slow as the slowest VJP computation (incurring a $t_c \cdot w \cdot C_{\text{max}}$ latency cost), while a single slow VJP computation would only incur a single $C_{\text{max}}$ latency cost for backpropagation.  
\end{enumerate}

Note that this is a separate and distinct proof to the result by \citep{rosenbaum_relationship_2022}, which showed that for a specific step size of 1, FPA-PC computes gradients in $t = L - \ell$ steps (i.e. in an equivalent number of steps to backpropagation). It was not clear however (1) whether FPA-PC could converge in fewer time steps for arbitrary step sizes and (2) whether FPA-PC inference had worse time complexity (or slower runtime), which was dependent on the time complexity of each inference step. Here we show that the FPA-PC convergence can provably never be faster than standard backpropagation for \emph{any} step size.

We test these results by enacting FPA-PC inference on multi-layer perceptron networks of various sizes, while varying the number of inference steps relative to the size of the network. We then compute the cosine similarity of gradient updates relative to those obtained via backpropagation. We find, as expected, that cosine similarity drops rapidly as the number of inference steps falls below the number of layers in the network, demonstrating a failure to converge (Figure \ref{fig:bp_cosine_similarity}). Inference learning rates lower than 1 show a failure to converge even for a number of inference steps higher than the number of layers in the network. We also report validation and test set accuracies, which experience significant drops as this occurs also (Figures \ref{fig:fpa_pc_val_accuracy} and \ref{fig:fpa_pc_test_accuracy}).

\begin{figure}
\centering
\includegraphics[width=\linewidth]{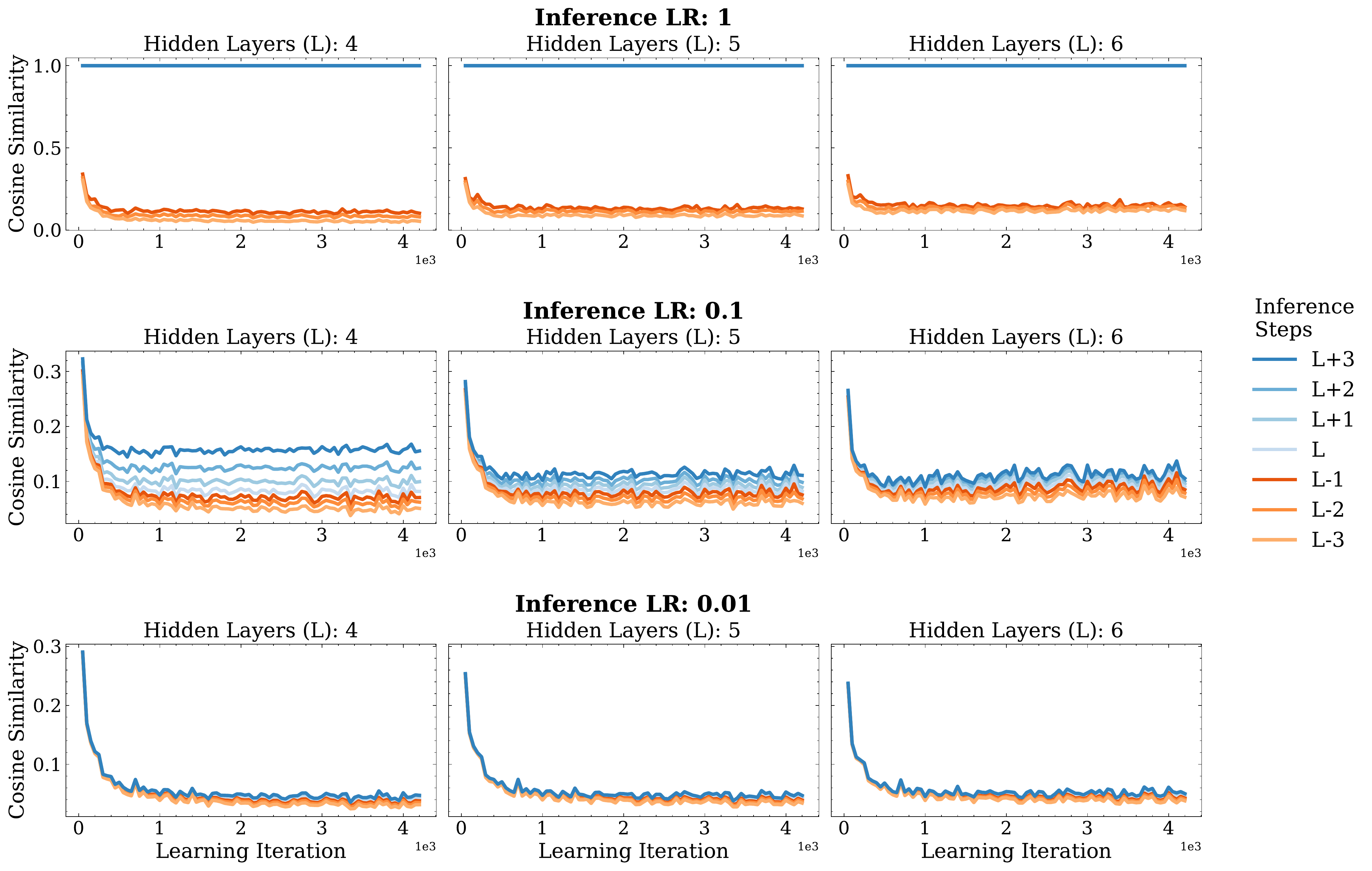}
\caption{Cosine similarity between FPA-PC learning updates and backpropagation on MNIST, tested for a varying number of inference steps, relative to the size of the network.}
\label{fig:bp_cosine_similarity}
\end{figure}

\begin{figure}
\centering
\includegraphics[width=\linewidth]{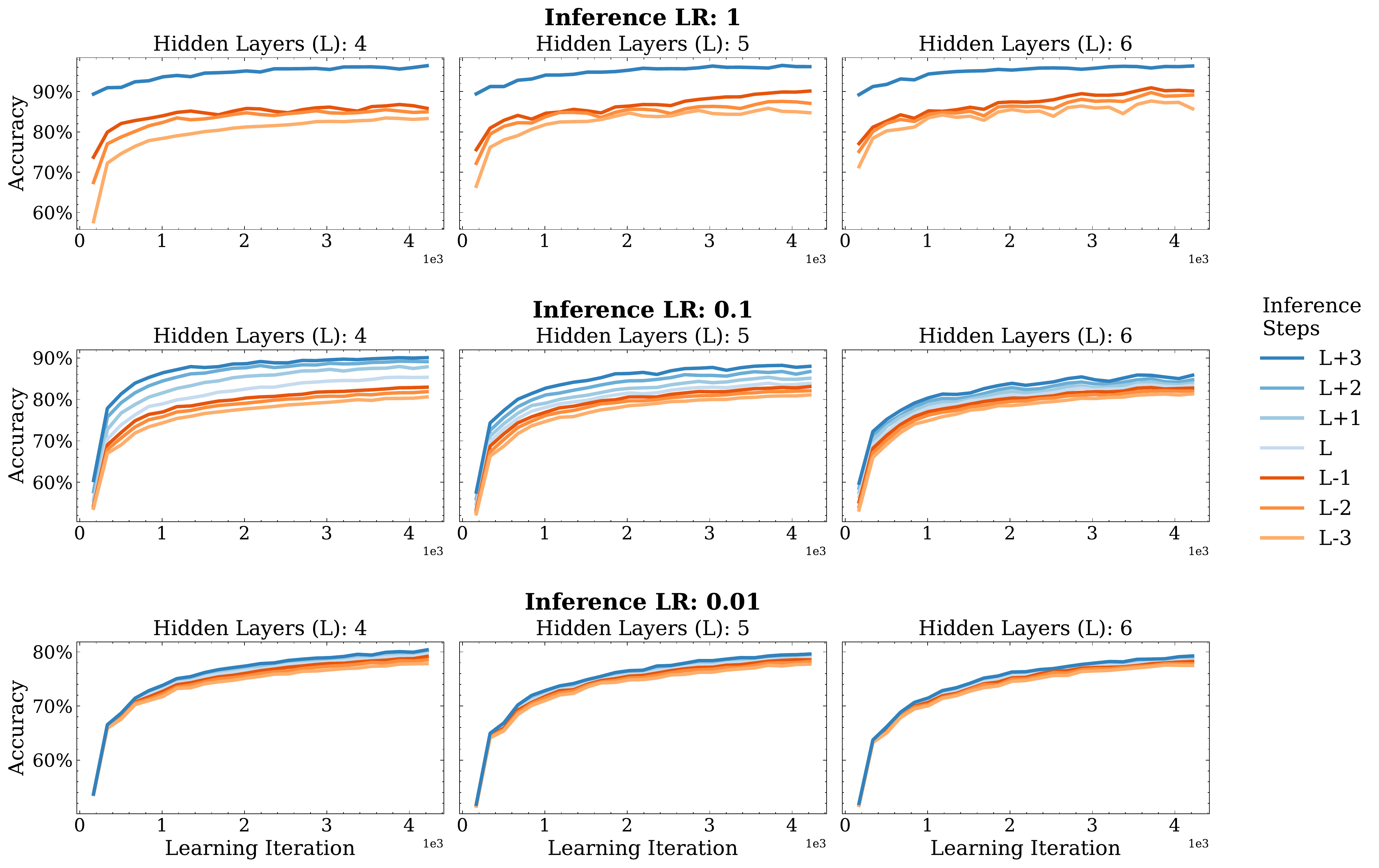}
\caption{Validation set accuracy on MNIST with FPA-PC for a varying number of inference steps, relative to the number of layers in the network.}
\label{fig:fpa_pc_val_accuracy}
\end{figure}

\begin{figure}
\centering
\includegraphics[width=\linewidth]{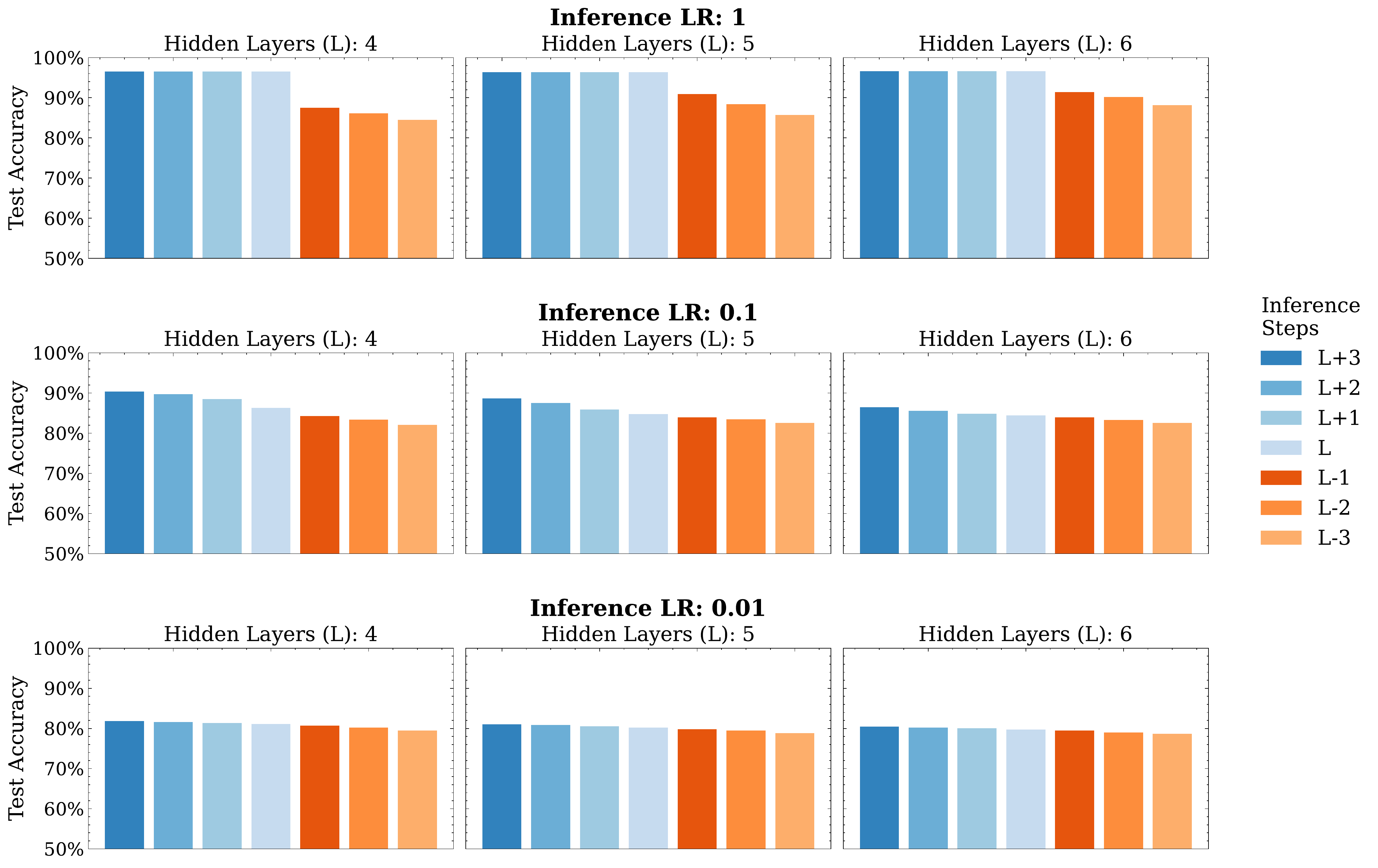}
\caption{Final test set accuracy on MNIST with FPA-PC for a varying number of inference steps, relative to the number of layers in the network.}
\label{fig:fpa_pc_test_accuracy}
\end{figure}

\newpage 

\subsection{Zero Divergence Inference Learning (Z-IL)}

Subsequent to the FPA-modified formulation of PC, further modified formulations of PC were presented in \citep{song_can_2020, salvatori_predictive_2021}. The resultant algorithm, termed \textit{Zero Divergence Inference Learning} (Z-IL), achieved a similar effect to FPA-PC without explicitly fixing the feed-forward values and Jacobians and did so by requiring a very specific set of changes. Specifically, the following set of changes to the standard PC algorithm were required:  

\begin{itemize}
    \item \textbf{Requirement 1: } (As before): The variational modes or latent states ($x_1, ..., x_{L-1}$) are initialised to the feed-forward values of the corresponding backpropagation-based computational chain ($\mu_1, ..., \mu_{L-1}$)
    \item \textbf{Requirement 2: } The inference learning rate $\gamma$ is set to 1
    \item \textbf{Requirement 3: } A particular layer is updated specifically and exclusively at a particular inference time-step ($t = L-\ell$) corresponding to its distance from the output node in the computational chain.
\end{itemize}

Ablation of any one of these three requirements was shown to result in losing equivalence between Z-IL and backpropagation\footnote{Note that, for the purposes of this proof, we will ignore FA-Z-IL (Fully Autonomous Z-IL), which builds upon Z-IL, as its primary purpose was to remove the neurobiologically implausible requirement that the weight update is triggered manually at a particular inference step, and thus does not impact the results in this section.}. 

This algorithm relies on the fact that at a particular time-step, the node corresponding to a particular layer $x_\ell$ is experiencing an update that is equivalent to that obtained under FPA-PC dynamics. This occurs because the current value of any particular node ($x_\ell$) does not deviate from the feed-forward values they were initialised with until after a particular (i.e. $t=L-\ell$) inference step (See Lemma A.3 from Supplementary Material in \citep{song_can_2020}). For the node corresponding to the weight being updated, this has the effect of mimicking the fixing of feed-forward values seen under the FPA modified form (Equations \ref{eq:fpa_fixing_first}-\ref{eq:fpa_fixing_last}), since both its predictions and the predictions of its parents ($\ell-1$) remain fixed to their feed-forward values. 

To see this explicitly, we describe the relevant computations occurring at each inference step for the example function $F$, in table \ref{tab:z_il_v_backprop}. We restrict ourselves at each inference step to only describing the changes that occur to $x_\ell$, $e_\ell$ and $\theta_\ell$ for every timestep $\ell$. Therefore, we will not record in the table below any computations occurring at nodes $>\ell$,  as they will have no future influence on any parameter learning, as well as nodes $<\ell$, as no change is occurring for that inference step. For notational simplicity we drop time indices, rather, values on the right-hand-side of the equations in Table \ref{tab:z_il_v_backprop} refer to those from the previous time step. 

\begin{table}[h]
    \centering
    \begin{tblr}{|X[2,$, c]|X[5, $, c]|X[5, $, c]|}
    \hline
    \text{\textbf{Step (T)}} & \text{\textbf{Z-IL}} & \text{\textbf{Backpropagation}} \\
    \hline
    
    T=0 & 
    \begin{tblr}{@{}Q[$]@{}Q[$]@{}}
    \end{tblr}
    e_L = \frac{\partial f_L}{\partial \mu_L} &
    e_L = \frac{\partial f_L}{\partial \mu_L} \\
    \hline
    
    T=1 &
    \begin{tblr}{@{}Q[$]@{}Q[$]@{}}
    x_{L-1} & = \underbrace{x_{L-1}}_{=\mu_{L-1}} + \frac{\partial f_{L-1}}{\partial \mu_{L-1}}^T e_{L} \\
    e_{L-1} & = x_{L-1} - \mu_{L-1}
    \end{tblr}
    &
    e_{L-1} = \frac{\partial f_{L-1}}{\partial \mu_{L-1}}^T e_{L}
    \\
    \hline
    
    \vdots & \vdots & \vdots \\
    \hline
    
    T=L-1 
    &
    \begin{tblr}{@{}Q[$]@{}Q[$]@{}}
    x_{1} & = \underbrace{x_1}_{=\mu_{1}} + \frac{\partial f_{1}}{\partial \mu_{1}}^T e_{2} \\
    e_{1} & = x_{1} - \mu_{1}
    \end{tblr}
    &
    e_{1} = \frac{\partial f_{1}}{\partial \mu_{1}}^T e_{2}
    \\
    \hline
    
    \end{tblr}
\label{tab:z_il_v_backprop}
\end{table}

We can see here that in each time step, for the computations relevant to parameter gradient updates, Z-IL engages in precisely the same computations as backpropagation (a VJP evaluated at the feed-forward values and the subsequent error), these are then added to a constant (the feed-forward activations $\mu_\ell$), before this constant is then subtracted. In addition to these computations, we note that Z-IL (and its variant FA-Z-IL by extension), engages in wasted computation for nodes $x_i, i > \ell$, which we do not depict in the walk-through of the computations above. These computations occur despite not resulting in, or contributing to, any parameter updates, and are thus unnecessary. 

Since the dynamics for Z-IL are a special case of the dynamics of standard PC, we may once again use Theorem \ref{theorem:1}, Corollary \ref{corollary:1} and identical reasoning to that in section \ref{sec:fpa_pc}, to once again obtain complexity bounds for Z-IL. Which we find are similarly lower-bounded by that of backpropagation. 
\begin{align}
    \text{TIME} \{\text{\textit{Z-IL}}(F)\}
    &\ge \sum_{i=0}^{L} \text{TIME} \{\text{\textit{forward}}(f_i) \} + \sum_{t=0}^{L} \text{TIME} \{\text{\textit{vjp}}(f_{\text{slow}}) \} \\
    &\ge w \sum_{i=0}^{L} \mathcal{C}_i + w \sum_{i=0}^{L} 2 \mathcal{C}_{\text{max}} \\ 
    &\ge w \sum_{i=0}^{L} \mathcal{C}_i + w \sum_{i=0}^{L} 2 \mathcal{C}_i \\
    &\ge \text{TIME}\{\text{\textit{backprop}}(F)\}
    \label{eq:zil_time_complexity_last}
\end{align}

\subsection{Generalised-IL (G-IL), Learning Rate Stability and Online Learning}

For completeness, we note that some recent work has suggested that Predictive Coding (PC) and related energy-based models, which do not yield backpropagation-equivalent updates, may nonetheless exhibit greater robustness to high learning rates and show less degraded performance in the online learning (batch size 1) settings \citep{alonso_theoretical_2022, song_inferring_2022}. More specifically, \citep{song_inferring_2022} demonstrates that a variant of PC with additional regularisation terms, called Generalised-IL (G-IL), may approximate implicit SGD, with the approximation becoming exact under certain limits. The authors of this work further present a modified algorithm (IL-prox), which guarantees equal updates to \textit{implicit} SGD. The resultant schemes demonstrated greater robustness to the learning rate and less degraded performance in online learning regimes. These properties are also demonstrated empirically, for energy-based models (EBMs) in general, as a consequence of a principle dubbed "prospective configuration" in \citep{alonso_theoretical_2022}.  

These findings are intriguing, as robustness to learning rates aligns with expectations for a backward Euler integration scheme -- the implicit counterpart to the forward Euler integration scheme from which SGD originates -- and an optimization approach that is demonstrably less prone to divergence in non-stochastic contexts.

It is unclear, however, whether the differences in learning rate stability would remain when testing against optimizers such as Adam, or SGD with momentum, both of which are well-understood in theory, and practice, to ameliorate the risk of divergence under high learning rates, by improving the conditioning of the loss landscape or dampening oscillations. Moreover, there is evidence to suggest that small-batch SGD itself has a greater robustness to learning rate, \citep{masters_revisiting_2018, shallue_measuring_2019} possibly due to the additional stochasticity in the small-batch setting preventing the accumulation of Euler integration errors. As such, the authors posit that it would be a compelling avenue for future research to explore whether the performance gap in online learning settings persists when online SGD is combined with high learning rates.

While these findings do not impact our presented results or the use of PC as a direct substitute for backpropagation, they raise thought-provoking questions on whether modified and unmodified PC forms may nonetheless possess attributes desirable for neuromorphic learning, even if they do not necessarily exhibit parameter updates identical to those under backpropagation.

\section{Conclusion and Key Results}

In this paper, we have investigated the computational efficiency, learning dynamics, and neurobiological plausibility of various PC variants in comparison to the backpropagation algorithm. We now summarise the key results and points made:

\begin{quote}
\begin{description}
    \item[Result 1] The infinite-variance limit modification from \citep{whittington_approximation_2017} (see Case 2 above), is equivalent to assuming a model with strictly deterministic (non-probabilistic) latent states. 
    
    \item[Result 2] Modified variants of PC: FPA-PC \citep{millidge_predictive_2020}, and Z-IL \citep{song_can_2020, salvatori_predictive_2021}, do not follow a free energy gradient. Corollary: learning does not correspond to learning of a latent probabilistic model via variational Bayes. 
    
    \item[Result 3] Modified variants of PC: FPA-PC, and Z-IL, are guaranteed to have worse time complexity than backpropagation, even when accounting for fully parallel computation within each inference step. (See Theorem \ref{theorem:1}, corresponding Corollary \ref{corollary:1} and Equations \ref{eq:fpa_time_complexity_first}-\ref{eq:zil_time_complexity_last})

    \item[Point 1] Equivalence, or approximate equivalence, of PC with backpropagation mandates adopting an inverted scheme, unlike traditional formulations of generative PC, which has non-trivial implications for neurobiological plausibility. 
    
    \item[Point 2] A naive implementation of standard PC does not suffer from the weight \textit{transport} problem in the same sense as that which is present in naive neurobiological proposals of backpropagation. It does, however, suffer from an analogous but potentially less implausible, weight \textit{symmetry} issue. 
\end{description}    
\end{quote}

These results raise doubt with respect to the advantages that recent variants of PC may provide with regards to backpropagation and its neuromorphic implementation, given that PC variants which result in equivalent or close to equivalent gradient updates also engage in precisely the same computations as backpropagation, and do so in much the same, equally local/non-local, way. This is despite the ostensibly parallelised computation of PC networks, which from the perspective of backpropagating errors merely results in additional, or unused, computation, and not faster error propagation.      

Furthermore, by introducing modifications to obtain, or approximate, equivalence, PC variants lose the generative variational Bayes interpretation of standard formulations of PC, and its various strengths, such as maintaining uncertainty over latent hidden states or causes, or parity with our current understanding of the organisational and functional properties of pathways within the cortex.

\section*{Acknowledgements}
The authors would like to thank Prof. Christopher L. Buckley and his group for their valuable comments and feedback on this manuscript.

\bibliographystyle{apacite}

\bibliography{references}

\end{document}